%% file: main.tex
\documentclass{ecai} 

\usepackage{latexsym}
\usepackage{amssymb}
\usepackage{amsmath}
\usepackage{amsthm}
\usepackage{booktabs}
\usepackage{enumitem}
\usepackage{graphicx}
\usepackage{color}
\usepackage{algorithm,algorithmic}

\newtheorem{theorem}{Theorem}
\newtheorem{lemma}[theorem]{Lemma}
\newtheorem{corollary}[theorem]{Corollary}
\newtheorem{proposition}[theorem]{Proposition}

\newtheorem{definition}{Definition}

\newcommand{\BibTeX}{B\kern-.05em{\sc i\kern-.025em b}\kern-.08em\TeX}

\begin{document}


\begin{frontmatter}


\paperid{1509} 


\title{Learning Linear Utility Functions From Pairwise Comparison Queries}

\author[A]{\fnms{Luise}~\snm{Ge}}
\author[A]{\fnms{Brendan}~\snm{Juba}}
\author[A]{\fnms{Yevgeniy}~\snm{Vorobeychik}} 

\address[A]{Washington University in St Louis}


\begin{abstract}
We study learnability of linear utility functions from pairwise comparison queries.
In particular, we consider two learning objectives. The first objective is to predict out-of-sample responses to pairwise comparisons, whereas the second is to approximately recover the true parameters of the utility function.
We show that in the passive learning setting, linear utilities are efficiently learnable with respect to the first objective, both when query responses are uncorrupted by noise, and under Tsybakov noise when the distributions are sufficiently ``nice''. 
In contrast, we show that utility parameters are not learnable for a large set of data distributions without strong modeling assumptions,  even when query responses are noise-free.
Next, we proceed to analyze the learning problem in an active learning setting. 
In this case, we show that even the second objective is efficiently learnable, and present algorithms for both the noise-free and noisy query response settings.
Our results thus exhibit a qualitative learnability gap between passive and active learning from pairwise preference queries, demonstrating the value of the ability to select pairwise queries for utility learning.
\end{abstract}

\end{frontmatter}


\input{intro}
\input{related_work}
\input{prelim}

\input{passive}

\input{active}



\section{Conclusion}
We presented an investigation of learnability of linear utility functions from pairwise comparison queries.
Our results consider both the passive and active learning problems, as well as two objects, the first involving \emph{prediction} while the second concerned with \emph{estimation}.
Overall, we find that estimation is generally more challenging than prediction, and active learning enables qualitatively better sample complexity in this case.
There are a number of directions to extend our work. For example, one can aim to generalize our results to list-wise comparisons described in \cite{zhao2019learning}. 
Specific to active learning, especially related to LLM and RLHF, we acknowledge the challenges of query synthesis, as well as the issue of inverting the embedding $\phi(x)$. 
Thus, there is considerable room for follow-up work in this regard.
Finally, our framework is restricted to linear utilities, even accounting for a non-linear embedding.
An important direction for future work is to consider more general classes of utility models, such as those using neural network architecture.





\bibliography{mybibfile}

\input{appendix}

\end{document}

%% file: intro.tex
\section{Introduction}
Guiding machines in accordance with human values is a fundamental principle in technology, commonly referred to as ``alignment" in the field of Artificial Intelligence (AI). 
A common approach for achieving alignment involves learning a utility function (or \emph{reward model}) from a large collection of human responses to \emph{pairwise comparison queries}, that is, queries about their preferences between pairs of options.
A notable example of this is the collection of such data in training large language models (LLMs) to increase helpfulness and reduce harm~\cite{ouyang2022training,bai2022training} as part of a \emph{reinforcement learning from human feedback (RLHF)} framework, in which learned reward models are used as reward functions in a reinforcement learning loop to fine-tune LLMs.
Indeed, utility models learned from pairwise comparison queries have been used in a broad array of domains quite apart from LLMs, such as to train human-aligned controllers~\cite{christiano2017deep}, develop human-aligned kidney exchange algorithms~\cite{freedman2020adapting}, and personalize recommendation systems~\cite{kalloori2018eliciting,qomariyah2018pairwise}.


Despite the centrality of utility function estimation from pairwise comparison data across a broad array of domains, understanding of what is learnable in this context remains relatively limited, even when utility functions are linear.
Theoretical consideration of this problem has been largely in the context of \emph{random utility model (RUM)} learning based on complete preference information about a fixed set of candidates (outcomes)~\cite{marschak1959binary,becker1963stochastic,negahban2018learning,noothigattu2018voting}.
Issues such as identifiability~\cite{zhao2020learning} and social choice properties of learned utility models~\cite{noothigattu2018voting,xia2011maximum,xia2019learning} have been extensively studied.
However, few have considered the issue of learning utility models over a vector space of outcomes, where observed preference data constitutes a negligible fraction of possible preference ranking data, and generalization beyond what is observed is key (a notable exception is Zhu et al.~\cite{pmlr-v202-zhu23f}, which we discuss below).
As such settings constitute an important new horizon in social choice as it is integrated into modern AI, a deeper understanding of the issue is needed.

We consider the problem of learning linear utility functions from pairwise comparison data.
Our fundamental question is \emph{under what conditions is sample-efficient learning possible}.
Aside from a recent positive result by \citeauthor{pmlr-v202-zhu23f} \cite{pmlr-v202-zhu23f}, who showed that parameters of a linear utility function can be efficiently estimated in the covariance norm when data follows the Bradley-Terry (BT) random utility model~\cite{bradley1952rank}, this question remains largely open.
We consider the problem for a general preference noise model of which BT, and many other common RUMs are a special case, with respect to two natural measures of generalization error: 
1) accuracy in \emph{predicting} responses to unseen pairwise comparison queries, and
2) efficacy (which we capture as Euclidean distance) in \emph{estimating} parameters of the utility function.

We first investigate the ``passive'' Probably Approximately Correct (PAC) learning model in which pairs of inputs (outcomes) are generated according to an unknown distribution $\mathcal{P}$, and responses to these comparison queries may be corrupted by noise that comes from a given distribution $\mathcal{Q}$.
We show that when there is no noise in query responses, we can indeed efficiently learn to predict responses to unseen pairwise comparison queries with respect to the true utility function in the PAC sense.
However, the presence of noise \emph{with any c.d.f. continuous at $0$}~leads to an exponential lower bound on sample complexity in this model.
Nevertheless, we show that if the input distribution is well-behaved and the query noise satisfies the Tsybakov noise condition, linear utilities are efficiently learnable in this model.

Turning next to the goal of estimating utility function parameters, we show that it is impossible to do so effectively with polynomial sample complexity \emph{even when there is no query noise}.
This is in contrast to the positive result by \citeauthor{pmlr-v202-zhu23f} \cite{pmlr-v202-zhu23f} which implies that we can efficiently learn linear utility functions when the smallest eigenvalue of the covariance matrix of the input distribution is bounded from below, and query responses follow the Bradley-Terry noise model.


While the ``passive'' setting is conventional in learning theory, the practice of reward model learning, particularly in the context of RLHF, is often to carefully select the outcome pairs for which human preferences are queried.
To study this systematically, we therefore consider an active learning problem whereby we can select arbitrary pairs of inputs to query interactively in order to learn the parameters of a linear utility function.
Our main results are that in the active learning model we can efficiently estimate parameters of the linear utility model, whether or not query responses are corrupted by noise.

%% file: related_work.tex
\section{Related Work}

\textbf{Learning Utility Models from Pairwise Comparisons:}
The fundamental framework that we build on is random utility model learning, in which utility information is provided indirectly through ranking (e.g., pairwise) comparisons~\cite{marschak1959binary,bradley1952rank,xia2019learning}.
A common approach for learning random utility models from pairwise comparisons is by using maximum likelihood estimation (MLE)~\cite{noothigattu2020axioms,xia2011maximum,xia2019learning}.
Unlike our setting, these typically investigate problems with a finite set of alternatives~\cite{shah2015estimation,negahban2018learning}.
Our focus, in contrast, is to learn a utility function over a continuous vector space.
This problem has been considered by \citet{pmlr-v202-zhu23f}, who show that linear utilities can be efficiently learned using maximum likelihood estimators in the Bradley-Terry noise model.

\noindent
\textbf{PAC Learning of Halfspaces With Noise:}
Common models of noise in random utility models have the property that the closer the pair's difference is to the halfspace defined by the linear utility function, the higher the chance that this comparison is flipped. This kind of noise has been previously studied as boundary consistent noise in \cite{du2015modelling,zhang2021learning,zhang2021improved}. The possibility of an arbitrarily small difference between noise and $\frac{1}{2}$ makes learning very challenging, even for learning halfspaces~\cite{balcan2020noise}. 
Consequently, proposals for tractable learning often rely on assumptions regarding distributions of noise.
The spectrum of such noise models starts from the easiest version called \textit{random classification noise (RCN)} \cite{angluin1988learning} -- where each label is flipped independently with equal probability \textit{exactly} $\eta < \frac{1}{2}$ -- where efficient algorithms exist~\cite{blum1998polynomial}. 
At the other end are malicious noise~\cite{kearns1988learning} and agnostic learning~\cite{kearns1992toward,klivans2014embedding,daniely2016complexity,diakonikolas2020near}, where an adversarial oracle can flip any small constant fraction of the labels, typically making these models less tractable. Additionally, active learning has been proposed as a solution to reduce the number of samples the learner needs to query before approximately learning the concept class with high probability. While active learning requires exponentially fewer labeled samples than PAC-learning for simple classes such as thresholds in one dimension, it fails in general to provide asymptotic improvement for broad classes such as halfspaces \cite{dasgupta2005coarse}.

\noindent
\textbf{Robust Parameter Estimation:} Learning a linear classifier has also been a classical problem beyond  learning theory, with \emph{empirical risk minimisation} the most popular paradigm, with a range of tools including Bayes classifiers, perceptron learning, and support vector machines (SVM). 
Most directly related is the connection between robustness (to input noise) and regularization in SVM~\cite{xu2009robustness}. 
However, this robustness is with respect to \emph{input} noise, whereas our consideration is noise in pairwise comparison responses (outputs).


%


%% file: prelim.tex
\section{Preliminaries\label{prelim}}

Our goal is to learn utility functions from pairwise comparison queries.
To this end, we consider the hypothesis class $\mathcal{U}$ of weight-normalized monotone linear utility functions $u(x) = w^T \phi(x)$ with respect to a known continuous embedding of candidate profiles $\phi(x): \mathbb{R}^d \xrightarrow{} \mathcal{X} \equiv [0,1]^m$, where $\mathbb{R}^d$ is the original space the candidates lie in,  and $w \ge 0$ and $\|w\|_1 = 1$.

In the passive learning setting, we assume that we obtain a dataset $\mathcal{D} = \{(x_i,x_i',y_i)\}_{i=1}^n$ which contains labels $y_i \in \{0,1\}$ associated with responses to pairwise comparison queries $(x_i,x_i')$ that are interpreted as preferences.
In particular, $y_i=1$ if $x_i'$ is preferred to $x_i$ (which we denote by $x_i' \succ x_i$, and $y_i=0$ otherwise.
We assume that input pairs $(x,x') \in \mathcal{D}$ are generated i.i.d.\ according to an unknown distribution $\mathcal{P}$ over $\mathbb{R}^d$.
This, in turn, induces a distribution $\mathcal{P}_\phi$ over embedded pairs $(\phi(x),\phi(x'))$.

We assume that the label $y$ is generated according to the commonly used \emph{random utility model (RUM)}~\cite{noothigattu2020axioms}. 
Specifically,
if $u(x)$ is the true utility function, query responses follow a random utility model $\tilde{u}(x) = u(x) + \Tilde{\zeta}$, where $\Tilde{\zeta}$ is an independent random variable distributed according to a fixed probability distribution $\tilde{\mathcal{Q}}$.
Define $Z_{\tilde{u}}(x,x') \equiv \mathrm{sign}(\tilde{u}(x')-\tilde{u}(x))$, where we define $\mathrm{sign}(z)=1$ if $z > 0$, $0$ if $z<0$, and a random variable being $0$ or $1$ with equal probability $0.5$ if $z=0$.
Then $y = Z_{\tilde{u}}(x,x') \in \{0,1\}$.

We further define the difference between a pair $(\phi(x),\phi(x'))$ as $\Delta_\phi(x) \equiv (\phi(x')-\phi(x)) \in [-1,1]^m$. Since $\tilde{u}(x')-\tilde{u}(x) = w^T \Delta_\phi(x) + \Tilde{\zeta}(\phi(x')) - \Tilde{\zeta}(\phi(x))$,
it will be most useful to consider $\zeta \equiv \Tilde{\zeta}(\phi(x'))-\Tilde{\zeta}(\phi(x))$ for a pair of feature vectors $(\phi(x),\phi(x'))$, so that $Z_{\tilde{u}}(x,x') = \mathrm{sign}(w^T \Delta_\phi(x) + \zeta)$.
Let $\mathcal{Q}$ be the distribution over $\zeta$ induced by $\tilde{\mathcal{Q}}$, and let $F(\zeta)$ be its c.d.f., which we assume is continuous on $\mathbb{R}$ except in the noise-free special case.
Then the probability that $y = 1$ for a pair $(\phi(x),\phi(x'))$, that is, the probability that $x' \succ x$, is $\Pr(x' \succ x) = \Pr(\zeta \le w^T \Delta_\phi(x)) = F(w^T \Delta_\phi(x))$.
As we assume $\Pr(x' \succ x) + \Pr(x \succ x')=1$, 
$F(w^T \Delta_\phi(x))+F(-w^T \Delta_\phi(x))=1$ leading to $F(x)+F(-x)=1$ and $F(0)=1/2$.

Note that our random utility model is quite general.
For example, 
the two most widely-used models are both its special cases: the Bradley-Terry (BT) model~\cite{bradley1952rank}, in which $F$ is the logistic distribution, and the Thurstone-Mosteller (TM) model~\cite{thurstone1927law}, in which $F$ is a Gaussian distribution.
Additionally, we consider an important special case in which there is no noise, i.e., $\zeta = 0$.
In this case, $\Pr(x' \succ x) \in \{0,0.5,1\}$ depending on whether $w^T\Delta_\phi(x)$ is negative, zero, or positive.

Broadly speaking, the goal of learning utility functions from $\mathcal{D}$ is to effectively capture the true utility $u(x)$.
In the setting where data provides only information about pairwise preferences, there are a number of reasonable learnability goals.
Here, we consider two.

Our first goal is motivated by a general consideration that a common role of a utility function is to induce a ranking over alternatives, used in downstream tasks, such as reinforcement learning, recommendations, and so on.
Thus, a corresponding aim of learning a utility function is to approximate its ranking over a finite subset of alternatives, so that the learned function is a useful proxy in downstream tasks.
Here, we consider the simplest variant of this, which is to learn a linear function $\hat{u}: \mathcal{X} \xrightarrow{} [0,1]$ from $\mathcal{D}$ with the property that it yields the same outcomes from pairwise comparisons as $u(x)$.
Formally, we capture this using the following error function:
\begin{equation}
    \label{E:e1}
    e_1(\hat{u},u) = \Pr_{(x,x') \sim \mathcal{P}} 
    \left(Z_{\hat{u}}(x,x') \ne Z_u(x,x')\right).
\end{equation}
We refer to this goal as minimizing the \emph{error of predicted pairwise preferences}.
Note that this error function has an important difference from conventional learning goals in similar settings: we wish to predict pairwise comparisons with respect to the true utility $u$, rather than the noise-perturbed utility $\tilde{u}$.
This is a consequential difference, as it effectively constitutes a distributional shift when pairwise preference responses are noisy.

Our second goal is to accurately capture the weights of $u$.
Formally, we define it as
\begin{equation}
    \label{E:e2}
    e_2(\hat{u},u) = \|\hat{w} - w^*\|_p^p
\end{equation}
where $p \ge 1$ and $w^*$ is the true weight.
Henceforth, we focus on Euclidean norm $\ell_2$.
Our results generalize to any $p \ge 1$: counterexamples in the impossibility results still work, and the positive results hold by norm equivalence in $\mathbb{R}^m$.
Moreover, focusing on $\ell_2$ norm facilitates a direct comparison to \citeauthor{pmlr-v202-zhu23f} \cite{pmlr-v202-zhu23f}, who use a seminorm $\|x\|_\Sigma$ where $\Sigma$ is the data covariance matrix that is most comparable with respected to the $\ell_2$.
We refer to this model as the \emph{utility estimation error}.

Our learnability discussion will be based on an adaptation of the \textit{Probably Approximately Correct (PAC) Learning} framework~\cite{valiant1984theory}, with the goal of identifying a utility model $\hat{u}$ that has error at most $\varepsilon$ with probability at least $1-\delta$ for $\varepsilon>0$ and $\delta > 0$.
We now formalize this as PAC learnability from pairwise comparisons (PAC-PC).
Let $\mathbb{N}$ be the space of natural numbers.

\begin{definition}[PAC-PC learnability]
Given a noise distribution $\mathcal{Q}$, a utility class $\mathcal{U}$ is PAC learnable from pairwise comparisons (PAC-PC learnable) for error function $e(\hat{u},u)$ if there is a learning algorithm $\mathcal{A}$ and a function $n_A: (0,1)^2 \rightarrow \mathbb{N}$ such that for any input distribution $\mathcal{P}$, and $\forall \varepsilon,\delta \in (0,1)$, whenever $n \ge n_A(\varepsilon,\delta)$, $\mathcal{A}(\{(x_i,x_i',y_i)\}_{i=1}^n)$ with $(x_i,x_i')\sim \mathcal{P}$ i.i.d.\ and $y_i=\mathrm{sign}(u(x_i')-u(x_i)+\zeta_i)$ for $\zeta_i \sim \mathcal{Q}$ i.i.d.\ returns $\hat{u}$ with probability at least $1-\delta$ such that $e(\hat{u},u)\le \varepsilon$.
Moreover, if $n_A$ is polynomial in $1/\varepsilon, 1/\delta , \mathrm{Dim}(\mathcal{U})$, where $\mathrm{Dim}(\mathcal{U})$ is the dimension of $\mathcal{U}$, then we say $\mathcal{U}$ is efficiently PAC-PC learnable.
\end{definition}
\noindent
We stated this definition quite generally, but our focus here is on the class $\mathcal{U}$ of linear functions, in which case Dim($\mathcal{U}$) = $m$.

%% file: passive.tex
\section{Passive Learning}

We begin with the conventional learning setting, in which input pairs $(x,x')$ in the training data are generated i.i.d.\ from $\mathcal{P}$, inducing a distribution $\mathcal{P}_\phi$ over $(\phi(x),\phi(x'))$. 
We refer to this as the \emph{passive learning} setting to distinguish it from \emph{active learning} that we consider below.
For convenience, we also define the distribution of $\Delta_\phi(x) \sim \mathcal{P}_\phi'$. 
First, we consider the goal of learning to predict results of pairwise comparisons.
Subsequently, we investigate the goal of learning to estimate parameters $w$ of the true utility function $u(x)=w^{*^T} \phi(x)$.

\subsection{Predicting Pairwise Preferences}

We begin by considering the error function $e_1$, that is, where the goal is to predict outcomes of pairwise comparisons.
In the noise-free setting, we can immediately obtain the following result as a direct corollary of learnability of halfspaces.
\begin{theorem}
\label{T:e1no-noise}
   Suppose $\zeta=0$. Then linear utility functions are efficiently PAC-PC learnable under the error function $e_1$.
\end{theorem}
\begin{proof}
Let $y=g(\Delta_\phi(x))=\mathrm{sign}(w^{*^T} \Delta_\phi(x))$.
Then PAC-PC learning under $e_1$ and $\mathcal{P}$ is equivalent to PAC learnability of halfspaces $g(\Delta_\phi(x))$ under $\mathcal{P}'_\phi$.
As shown by \citeauthor{blumer1989learnability}~\cite{blumer1989learnability}, halfspaces are learnable with sample complexity $\mathcal{O}(\frac{1}{\varepsilon}(m+\log(\frac{1}{\delta})))$.
\end{proof}

With this connection between learnability in the sense of $e_1$ and learning of halfspaces in mind, we now consider the case in which pairwise preference responses are corrupted by noise.
Surprisingly, adding noise to pairwise preferences makes learning impossible for a natural broad class of distributions.
For the next impossibility result, as well as the discussion that follows, it will be useful to define $\eta(\Delta_\phi(x)) \equiv F(-|w^{*^T} \Delta_\phi(x)|)$, the chance of getting a misreported comparison label for $\Delta_\phi(x)$.
\begin{theorem}
\label{T:e1noiseneg}
 $\mathcal{U}$ is not efficiently PAC-PC learnable under error function $e_1$ if the preference noise distribution has a  c.d.f.~$F$ continuous at zero.
\end{theorem}
\begin{proof}
Following the discussion in Section~\ref{prelim}, we know that $\Pr(y=1)=F(w^{*^T} \Delta_\phi(x)), \Pr(y=0)=F(-w^{*^T} \Delta_\phi(x))=1-F(w^{*^T} \Delta_\phi(x))$ whereas the true label is $\mathrm{sign}(w^{*^T} \Delta_\phi(x))$. Hence $\eta(\Delta_\phi(x)) \equiv F(-|w^{*^T} \Delta_\phi(x)|)$ is indeed the chance of getting a misreported label for $\Delta_\phi(x)$.

Due to the continuity of $F$ at $0$, for any $t>0$, there exists $s>0$ such that for all $|v-0|\leq s$, $|F(v)-F(0)|\leq t$. As $F$ is monotonically increasing, and $-|w^{*^T} \Delta_\phi(x)|<0$, $\eta(\Delta_\phi(x))<F(0)=\frac{1}{2}$. Hence for a fixed $\varepsilon<1$, we can pick $t_0=1/\exp(1/\varepsilon)$. Then there exists a corresponding $s_0$ such that for a  $(\phi(x),\phi(x'))$ with $\Delta_\phi x$ of norm at most $s_0$, for all $w$, $|w^T\Delta_\phi(x)-0|\leq s_0$, $|\eta(\Delta_\phi(x))-\frac{1}{2}|\leq t_0.$ 
Now consider the oracle who always chooses one such pair $(\phi(x),\phi(x'))$.
With fixed probability at most $\frac{1}{2}+t_0$ it outputs the true label and with probability $\frac{1}{2}-t_0$ it outputs the false label.
In order to establish the Binomial distribution's confidence over $1-\delta$, we need sample complexity $\tilde{\Omega}(\frac{1}{t_0})=\tilde{\Omega}(\exp(\frac{1}{\varepsilon}))$, which is exponential in $\frac{1}{\varepsilon}$.\end{proof}

 Notably, all common models of noisy preference responses, such as Bradely-Terry and Thurstone-Mosteller, entail continuity of $F$ at $0$, so this result yields impossibility of learning linear utilities for these standard RUM settings.

The consequence of Theorem~\ref{T:e1noiseneg} is that to achieve any general positive results we must constrain the distribution over inputs and resulting embeddings.
We now show that we can leverage Tsybakov noise and a condition that the distribution over embeddings is well-behaved to attain sufficient conditions for efficient learnability.

\begin{definition}[Tsybakov Noise Condition]\cite{diakonikolas2021efficiently}
     Let $C$ be a concept class of Boolean-valued functions over $X = \mathbb{R}^d$, $F$ be a family of distributions on $X$, $0 < \varepsilon < 1$ be the error parameter, and $0 \leq \alpha < 1$, $A > 0$ be parameters of the noise model.
Let $f$ be an unknown target function in $C$. A Tsybakov example oracle, $EX^{Tsyb}(f,F)$, works
as follows: Each time $EX^{Tsyb}(f,F)$ is invoked, it returns a labeled example $(x,y)$, such that: (a) $x \sim \mathcal{D}_x$, where $\mathcal{D}_x$ is a fixed distribution in $F$, and (b) $y = f(x)$ with probability $1-\eta(x)$ and $y=1-f(x)$ with probability $\eta(x)$. Here $\eta(x)$ is an unknown function that satisfies the Tsybakov noise condition with parameters $(\alpha, A)$. That is, for any $0 < t \leq \frac{1}{2}$, $\eta(x)$ satisfies the condition $\Pr_{x \sim \mathcal{D}_x} [\eta(x) \geq \frac{1}{2}-t] \leq A t^{\frac{\alpha}{1-\alpha}}$. \end{definition}

\begin{definition} [Well-Behaved Distributions]\cite{diakonikolas2021efficiently}
For $L,R,U>0$ and $k \in \mathbb{Z}_{+}$, a distribution $\mathcal{D}_x$ on $\mathbb{R}^d$ is called $(k,L,R,U)$-well-behaved if for any projection ${(\mathcal{D}_x )}_V$ of $\mathcal{D}_x$ on a $k$-dimensional subspace $V$ of $\mathbb{R}^d$, the corresponding p.d.f. $\gamma_V$ on $V$ satisfies the following properties: (i) $\gamma_V(x)\geq L$, for all $x \in V$ with $||x||_2 \leq R$ (anti-anti-concentration), and (ii) $\gamma_V(x)\leq U$ for all $x \in V$ (anti-concentration). If, additionally, there exists $\beta \geq 1$ such that, for any $t > 0$ and unit vector $w \in \mathbb{R}^d$ , we have that $\Pr_{x\sim \mathcal{D}_x} [| \langle w, x\rangle | \geq t] \leq exp(1-t/\beta)$ (subexponential concentration), we call $\mathcal{D}_x (k,L,R,U,\beta)$ well-behaved.\end{definition}

\begin{theorem}[Learning Tsybakov Halfspaces under Well-Behaved Distributions] \cite[Theorem 5.1]{diakonikolas2021efficiently}
Let $\mathcal{D}_x$ be a $(3,L,R,U,\beta)$-well-behaved isotropic distribution on $\mathbb{R}^d \times \{\pm1\}$ that satisfies the $(\alpha,A)$-Tsybakov noise condition with respect to an unknown halfspace $f(x)=sign(\langle w,x\rangle)$. There exists an algorithm that draws $N=\beta^4(\frac{dUA}{RL\varepsilon})^{O(1/\alpha)}log(1/\delta)$ samples from $\mathcal{D}_x$, runs in $poly(N,d)$ time, and computes a vector $\Hat{w}$ such that, with probability $1-\delta$, we have $\Pr_{x\sim \mathcal{D}_x}[h(x) \neq  f(x)] \leq \varepsilon$.
\end{theorem}

We can leverage this result to obtain an efficient learning algorithm to predict pairwise comparisons for a broad range of distributions.

\begin{theorem}{\label{well-behaved}}
   Suppose that $\mathcal{P}_\phi$ is a $(3,L,R,U,\beta)$-well-behaved isotropic distribution, and the noise c.d.f.\ $F^{-1}(\zeta) \le \mathrm{poly}(\zeta)$ on $(0,\frac{1}{2}]$. Then the linear utility class $\mathcal{U}$ is efficiently PAC-PC learnable. 
\end{theorem}
\begin{proof}
Since $\mathcal{P}_\phi$ is $(3,L,R,U,\beta)$-well-behaved over $(\phi(x),\phi(x'))$, we consider a special 3-dimensional subspace $V$ with its basis being $u_1=(w^*,w^*), u_2=(v_2,v_2),u_3=(v_3,v_3)$, where $v_2$,$v_3$ are two orthonormal vectors lying on the hyperplane defined by our utility parameter $w^*$. Since $w^{*^T}v_2=w^{*^T}v_3=v_2^Tv_3=0$, these three vectors are linearly independent and form a basis.

Recall that for any point $\pi(x)=(\phi(x),\phi(x'))$ viewed as in the product space $\mathcal{X}^2$, its projected coordinates in $V$ is calculated as
$\pi(x)_V= (u_1^T \pi(x))  \frac{u_1}{||u_1||}+ (u_2^T \pi(x)) \frac{u_2}{||u_2||}+(u_3^T \pi(x)) \frac{u_3}{||u_3||}= ( u_1^T\pi(x), u_2^T\pi(x), u_3^T\pi(x)).$

First, by the subexponential concentration property, for any fixed $h>0$,  $\Pr_{\pi(x)\sim \mathcal{P}_\phi} \left( |u_2^T\pi(x)| \ge  h \right) < \exp(1- h/\beta) $, and  $\Pr_{\pi(x)\sim \mathcal{P}_\phi} \left( |u_3^T\pi(x)| \ge  h     \right) < \exp(1- h/\beta) $. So with probability at most $2\exp(1- h/\beta)$, the second or the third coordinate of the projection is not within $[-h,h].$

Next, by the anti-concentration property, the p.d.f. of the projection  $\pi(x)_V$ is bounded by a constant $\gamma_{V}(\pi(x))\le U$. Hence the total probability of the projection $\pi(x)_V$'s first coordinate between $[-h,h]$ is 
 $ Pr_{\pi(x) \sim \mathcal{P}_\phi}  \left(  |u_1^T\Delta_\phi(x)| \le h \right) \le  \int_{-h}^{h}   \int_{-h}^{h}   \int_{-h}^{h} U dt_1dt_2 dt_3+2  \exp(1- h/\beta)= 8Uh^3+  \exp(1- h/\beta)$. Since for any $\phi(x)<1$, $\exp(\phi(x))<1+\phi(x)+\phi(x)^2$, and $1-h/\beta <1$,  $\exp(1- h/\beta)< 1+ ((1- h/\beta)+ (1- h/\beta)^2$. So 
  $ Pr_{\pi(x) \sim \mathcal{P}_\phi}  \left( |u_1^T\Delta_\phi(x)| \le h \right)\le p_1(h) $ has a degree-$3$ polynomial upper bound.
 
Now, consider a $1-D$ projection of this subspace through $\text{Proj}(\phi(x),\phi(x'))=\phi(x')-\phi(x)$. We get $\text{Proj}(V)=\{(w^{*^T}\Delta_\phi(x), v_2^T\Delta_\phi(x), v_3^T\Delta_\phi(x))\}$.
Due to the triangle and Cauchy-Schwartz inequality, 
$|w^{*^T}\Delta_\phi(x)|=|w^{*^T}\phi(x')-w^{*^T}\phi(x)|\le |w^{*^T}\phi(x')| + |w^{*^T}\phi(x)| \le \sqrt{2}|u_1^T
\pi(x)|$. Hence, the probability of the first coordinate of $P(V)$, i.e. the margin $w^{*^T}\Delta_\phi(x)$ being between $[-\sqrt{2}h,\sqrt{2}h]$ has a polynomial upper bound:
$\Pr_{\Delta_\phi(x)\sim \mathcal{P}'_\phi} \left( |w^{*^T}\Delta_\phi(x) | \le  \sqrt{2}h     \right) \le \Pr_{\pi(x)\sim \mathcal{P}_\phi} \left( |u_2^T\pi(x)| \ge  h \right) < p_1(h).$

Because $\Pr_{\Delta_\phi(x) \sim \mathcal{P}'_\phi} \left(  \eta(\Delta_\phi(x)) \ge \frac{1}{2}-t \right)=\Pr_{\Delta_\phi(x) \sim \mathcal{P}'_\phi}  \left(  |w^{*^T}\Delta_\phi(x)| \le F^{-1}(\frac{1}{2}-t) \right)$,
as long as $F^{-1}(\frac{1}{2}-t)$ has a polynomial upper bound $p_2(t)$, we could establish another polynomial upper bound
$\Pr_{\Delta_\phi(x) \sim \mathcal{P}'_\phi}  \left(  |w^{*^T}\Delta_\phi(x)| \le F^{-1}(\frac{1}{2}-t) \right)<  p_1(\frac{F^{-1}(\frac{1}{2}-t)}{\sqrt{2}}) \le p_1 (\frac{p_2(t)}{\sqrt{2}})\in poly (t).   $


In other words, we could bound $\Pr_{\Delta_\phi(x) \sim \mathcal{P}'_\phi}  \left(  |w^{*^T}\Delta_\phi(x)| \le F^{-1}(\frac{1}{2}-t) \right) \le At^{\frac{\alpha}{1-\alpha}} $ by taking the leading coefficient of $p_1 (\frac{p_2(t)}{\sqrt{2}})$ being $A$, and $\frac{\alpha}{1-\alpha}$ being the degree of $p_1 (p_2(t))+1$.
Our noise model satisfies the $(\alpha,A)-$Tsybakov noise condition, and the algorithm in \cite{diakonikolas2021efficiently} applies.
\end{proof}

Next we show that the standard models of noise in RUM settings---Bradley-Terry~\cite{bradley1952rank} and Thurstone-Mosteller~\cite{thurstone1927law}---both satisfy the condition on the noise c.d.f.~in Theorem~\ref{well-behaved}.

\begin{proposition}
    The inverse of the c.d.f.~for the Bradley-Terry model satisfies $F^{-1}(x) \le \mathrm{poly}(x)$ on $(0,\frac{1}{2}]$.
\end{proposition}
\begin{proof}
    The inverse of the standard logistic function $F(x)=\frac{1}{1+\exp(-x)}$ is the logit function
    $F^{-1}(x)=logit(x)=\ln(\frac{x}{1-x})$.
    Because the derivative of the logit function $\frac{1}{x-x^2}$ is monotonically decreasing from $\infty$ to $4$ on $(0,\frac{1}{2}]$, the logit function is concave. Hence, it is bounded above by its gradient at $x=\frac{1}{2}$, which is $4x-2$. We have found a polynomial upper bound for $F^{-1}(x) \le 4x-2$ for $x \in (0,\frac{1}{2}]$.
\end{proof}

\begin{proposition}
   The inverse of the c.d.f.~for the Thurstone-Mosteller model satisfies $F^{-1}(x) \le \mathrm{poly}(x)$ on $(0,\frac{1}{2}]$.
\end{proposition}
\begin{proof}
     The inverse of the standard Gaussian c.d.f function $F(x)=\frac{1}{2}(1+\text{erf}(\frac{x}{\sqrt{2}}))$ is 
    $F^{-1}(x)=\sqrt{2}\text{erf}^{-1}(2x-1)$, where $\text{erf}$
  is the error function.
  Again, as the derivative of the inverse error function is 
  $\frac{\sqrt{\pi}}{2}\exp([\text{erf}^{-1}(x)]^2)$ , which is
  monotonically decreasing on $(-1,0]$, the inverse error function is concave and bounded above the gradient at $x=0$, which is $\frac{\sqrt{\pi}}{2}x$. Hence
$F^{-1}(x)=\sqrt{2}\text{erf}^{-1}(2x-1) \le \frac{\sqrt{2\pi}}{2}(2x-1)=\sqrt{2\pi}x- \frac{\sqrt{2\pi}}{2}$ for $x \in (0,\frac{1}{2}]$.
\end{proof}

\subsection{Estimating Utility Parameters\label{estimate}}

Next, we tackle the more challenging learning goal represented by the error function $e_2$, that is, where our goal is to learn to effectively estimate the parameters $w$ of the true linear utility model in the $\ell_p$ sense (focusing on $\ell_2$ here for clarity of exposition).

We begin with the known positive result.
Specifically, \citeauthor{pmlr-v202-zhu23f}~\cite{pmlr-v202-zhu23f} showed that maximum likelihood estimation (MLE) with the common BT noise model achieves efficient utility estimation in the following sense. 
With probability at least $1-\delta$, the MLE parameter estimate $\hat{w}$ from $n$ samples can achieve a bounded error measured in a seminorm with respect to $\Sigma= \frac{1}{n}\sum_{i=1}^n (\phi(x_i)-\phi(x'_i)_i) (\phi(x_i)-\phi(x'_i))^T$. With $C$ being constant,
\begin{equation}\label{zhu}
    ||\hat{w}-w^*||_{\Sigma} = \sqrt{(\hat{w}-w^*)^T\Sigma(\hat{w}-w^*)} \leq C \cdot \sqrt{\frac{m+\log(1/\delta)}{n}}.
\end{equation}


We are able to extend their positive result to an even bigger class of RUMs. The proof is deferred to the appendix.

\begin{theorem}\label{thm:generalisezhu}

Consider the loss function, $\ell_{\mathcal{D}}(w) = -\frac{1}{n}\sum_{i=1}^n \log \left(1(y^i=1) \cdot \Pr(x'\succ x) +  1(y^i=0) \cdot \Pr(x \succ x')\right).$
If there exists $\gamma>0$  such that the noise  c.d.f. $F(z)$ satisfies $F'(z)^2 - F''(z) \cdot F(z) \ge \gamma$ for all $z$, then with probability $1-\delta$, the MLE estimator $\hat{w}$ for $\ell_{\mathcal{D}}(w)$ can achieve inequality (\ref{zhu}).\end{theorem}

Consequently, as long as $\Sigma$'s smallest eigenvalue is bounded from below (which excludes the case where $\phi(x),\phi(x')$ are consistently close, as $\Sigma$ tends to 0 the bound becomes vacuous), 
sample complexity with respect to $e_2$ will also be quadratic.
Formally, suppose that the smallest eigenvalue  $\lambda_{min}\ge C_2>0$.
Then,
\begin{align*}
||\hat{w}-w^*||_2&=\sqrt{\frac{C_2}{C_2}(\hat{w}-w^*)^T(\hat{w}-w^*)}\\
&\le\sqrt{\frac{1}{C_2}(\hat{w}-w^*)^T\Sigma(\hat{w}-w^*)}\\ 
&\le \overline{C} \cdot \sqrt{\frac{m+\log(1/\delta)}{n}}
\end{align*}
for a constant $\overline{C}$.

Now we show that \emph{without any structural noise assumption,} learning to estimate linear utility parameters in the $\ell_2$ norm is impossible.
\begin{theorem}
    If there is no noise, the class of linear utility functions $\mathcal{U}$ is not always passively learnable, i.e.
     there is a probability distribution $\mathcal{P}$ over $(x,x')$, that
     no learning algorithm can achieve
     $e_2 \leq \varepsilon$ with probability at least $1-\delta$, for any
     $0 <\varepsilon,\delta < 1$ under a finite set of samples.
\end{theorem}
\begin{proof}
Let $\phi(x)=x$, and 
consider an oracle that only provides pairs $(x,x')$ with all coordinates satisfying $x'^{i}>x^i$ on each dimension $i$ . 
Then since $w\geq 0$, $w^T(x'-x)\geq 0$ for every $(x,x')$; in other words, the labels are uninformative and any algorithm produces the same distribution on guesses $\hat{w}$ for all $w$.
For $\delta<\frac{1}{2}$, we can consider $w,w'\in \mathcal{U}$ with distance equal to $\sqrt{2}$. This is achievable because $\mathcal{U}$ contains segments of length $\sqrt{2}$ defined by $\{w^i+w^j=1, 0 \le w^i,w^j\le 1\}$ for any distinct pair of $i,j \in \{1,\dots,m\}$ and we know $m\ge 2$ for otherwise there is no need to learn. But the algorithm cannot output $\hat{w}$ with $e_2<(\frac{\sqrt{2}}{2})^{2}=\frac{1}{2}$ for both $w$ and $w'$ simultaneously, and hence its output must obtain $e_2\geq \frac{1}{2}$ with probability $\geq \frac{1}{2}$ for one of these.\end{proof}

The reason behind this contrast between the positive and negative results lies precisely in the assumption of the noise model.
Unlike a common intuition that noise makes learning more challenging, a highly structured model assumption like BT actually provides more information to each query. 
From a pair where $\Delta_\phi(x)>0$, if we can estimate $\Pr(x' \succ x)$, since $\Pr(x' \succ x)=F(w^{*T}\Delta_\phi(x)), $ we obtain information about $w^*$ immediately.

\begin{figure}[H]
    \centering
    \includegraphics[width=5cm]{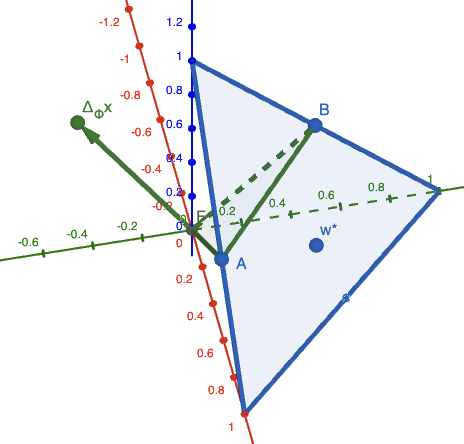}
    \caption{From the label we know $w^*$ is below the hyperplane of $\Delta_\phi(x)$.        }
\end{figure}

To illustrate a necessary and sufficient results for passive learning in the noise-free or noise-agnostic setting, we consider the problem from a geometric perspective.
It is helpful to recall a learning theory concept named \emph{version space}. A version space consists of all ``surviving'' hypotheses that are consistent with the conclusions we can draw from all labels observed so far. 
In our case, our version space denoted as $\mathcal{W}$ will be initialised the same as $\mathcal{U}$.
Now as each datapoint comes along,  the label $y$ gives us an (inaccurate) sign of
$w^{*T} \Delta_\phi(x)$, indicating whether the vector $w^*$ that we are searching for is above, on, or below the hyperplane defined by its normal vector $\Delta_\phi({x)}$. The version space will update accordingly.
For example, when $m=3$, the initial version space (hypothesis class) is depicted as the triangle $\{w^1+w^2+w^3-1=0 \ | \ 0 \leq w^1,w^2,w^3 \leq 1\}$ in Figure 1, with the true parameter labeled as $w^*$. Then suppose we receive a noiseless datapoint labeled with $y<0$, we can infer $w^*$ is below the segment $AB$ drawn as in Figure 1 since that is the space $\Delta_\phi(x)$ is normal to.


Thus, in order to achieve the second learning goal without a model assumption, we have to wish for that the incoming datapoints can help shrink the version space into 
a $\varepsilon$-radius ball centred at $w^*$. The difficulty with respect to $e_2$ is exactly that we don't have any control over the incoming information. Hence we will see from the following section that active learning is much more advantageous.

%% file: active.tex
\section{Active Learning}

Thus far, we noted that passive learning is generally quite challenging, particularly when responses to pairwise comparison queries are noisy.
However, in many practical settings, most notably, RLHF, utility (reward) models are \emph{not} learned passively.
Rather, one first carefully constructs and curates a dataset of input pairs, and only then obtains preference data for these~\cite{bai2022training,christiano2017deep}.
This is far more like an \emph{active learning} approach to the problem, for which there has been, to our knowledge, no prior work on learnability in the pairwise comparison model we consider here.
In this section, we demonstrate that the active learning model in this setting is qualitatively different from passive in that we now obtain strong efficient learnability results.

Our investigation falls under the umbrella of an interactive approach, which is considered a more efficient framework than pool and stream-based sampling techniques \cite{ling2008active,alabdulmohsin2015efficient,wang2016noise,chen2017near}. 
A common argument against query synthesis methods is that the artificial queries we thereby generate are uninterpretable. 
However, there is evidence that generating highly informative artificial training instances for tasks like text classification is feasible~\cite{schumann-rehbein-2019-active,piedboeuf2022working}. 
Although practical tricks may need to be developed to apply active learning approaches in a practical setting, these can in general save considerable annotation labor.
The second criticism is that if we do not have a distribution in mind, our learning may have limited value in terms of generalization. 
In any case, our study of active learning from a theoretical perspective also lends intuition as to why in practice there is considerable effort devoted to generating query inputs for learning utility models.

We structure our investigation in the same way as the passive learning case, considering first the problem of learning to predict responses to pairwise preference queries (the $e_1$ error model), and subsequently dealing with the more challenging problem of estimating utility parameters.

\subsection{Predicting Pairwise Preferences}

Just as in the passing learning setting, the positive results for active learning of linear utilities from pairwise comparisons in the $e_1$ sense
follow directly from known results for learning halfspaces.
In the noise-free case,
\citeauthor{alabdulmohsin2015efficient} \cite{alabdulmohsin2015efficient} and \citeauthor{chen2017near} \cite{chen2017near} provide efficient learning algorithms through query synthesis. 
\citeauthor{zhang2021improved} \cite{zhang2021improved}, in turn, address the setting with Tsybakov noise.
Consequently, we focus on the more challenging problem of estimating utility parameters in the active learning setting.
We tackle this problem next.


\subsection{Estimating Utility Parameters}


Continuing from our discussion in Section~\ref{estimate},
the goal in active learning is to ensure that in each step we query the most informative comparison pair inside the current version space in order to reduce the  size of the version space as fast as possible.  
We first provide an efficient algorithm for this in the noise-free setting, and subsequently in the case with query noise.
Here, a key challenge is an ability to invert the embedding $\phi(x)$.
Since our focus is on \emph{query} or \emph{sample} complexity, our results hold whether or not computing an inverse of $\phi$ (or, equivalently, a zero of $\phi(x)-v$ for a given $v$) is efficient (note that we do not need uniqueness); however, this does, of course, impact \emph{computational} complexity of the algorithms.
In special cases, such as if $\phi(x)=x$, or if $\phi(x)$ is affine (where an inverse can be computed using linear programming), computational complexity is polynomial as well, and more generally, we can in practice use gradient-based methods (such as Newton's method) to approximately find a zero of $\phi(x)-v$.

\begin{algorithm}[h]
    \caption{Noise-Free Active Learning}
    \label{alg:noisefreeal}
    \textbf{Input}: dimension of the instance space $m$, error bound $\varepsilon$ \\
    \textbf{Output}: $\hat{u}$ with $e_2(\hat{u},u) \leq \varepsilon$ 
    \begin{algorithmic}[1] 
    \STATE Initialize the Cartesian coordinate system $\mathcal{C}_1$ for $\mathbb{R}^m$
    \STATE Initialise the version space $\mathcal{W}=\{ w\in \mathbb{R}^{m} \ | \ w\ge0, ||w||_1=1\}$
    \STATE Initialize a separate $m-1$ dimensional Cartesian coordinate system $\mathcal{C}_2$ for $\mathcal{W}$
    \STATE $hypercube \gets []$ 
    \FOR{$i =1,\dots,m-1$}
    \STATE Let $s$ be the length of the version space $\mathcal{W}$ along the $i$-th axis of $\mathcal{C}_2$
    \WHILE{$s>2\varepsilon/\sqrt{m-1}$}
    \STATE Let $h_0,h_2$ be two $m-2$-dimensional hyperplanes tangent to the top and bottom of $\mathcal{W}$ along the $i$-th axis of $\mathcal{C}_2$
    \STATE Let $h_1$ be the $m-2$-dimensional hyperplane cutting through the middle of $\mathcal{W}$ along the $i$-th axis of $\mathcal{C}_2$
    \STATE Let $v$ be the outward pointing normal vector of the $m-1$-dimensional hyperplane in $\mathbb{R}^m$ consisting of $h_1$ and the origin of $\mathcal{C}_1$
        \STATE Let $x=0$ 
        \STATE Compute the inverse of $ \phi(0)+v$ as $x'=\Tilde{\phi}^{-1}(\phi(x)+v)$
        \STATE Let $h$ be the
        $m-2$ hyperplane in $\mathbb{R}^m$ as the intersection between $\mathcal{W}$ and the $m-1$ hyperplane normal to $\phi(x')-\phi(x)$
        \STATE Ask the oracle about $(x,x')$
        \IF {the label $y=1$}
        \STATE $bounds \gets \{h_0,h\}$
        \STATE $\mathcal{W}\gets$ the half of $\mathcal{W}$ between $h_0$ and $h$
        \ELSE
        \STATE $bounds \gets \{h,h_2\}$
        \STATE $\mathcal{W}\gets$  the half of $\mathcal{W}$ between $h$ and $h_2$
        \ENDIF
        \STATE $s\gets s/2$        
    \ENDWHILE
    \STATE $hypercube$.append($bounds$)
    \ENDFOR
    \RETURN the center of $hypercube$ 
    \end{algorithmic}
\end{algorithm}

The concrete active learning approach for noise-free learning of utility function parameters from pairwise queries is provided as Algorithm~\ref{alg:noisefreeal}.
Next, we prove that this algorithm in fact achieves efficient learning.
\begin{theorem}\label{thm:alg-query-bound}
 Suppose $\zeta = 0$ and we can approximate the inverse of $\phi$ with $\Tilde{\phi}^{-1}$ up to arbitrary precision, i.e. $||\phi(\Tilde{\phi}^{-1}(x))-x||_{\inf} <c$ for any constant $c$. Then for any $\varepsilon$, there is an active learning algorithm that returns a linear hypothesis $\hat{u}$ with $e_2(\hat{u},u)\leq \varepsilon $ after asking the oracle $\mathcal{O}(m \log(\frac{\sqrt{m}}{\varepsilon}))$ queries.
\end{theorem}
\begin{proof} We would like to trap the true weight $w$ inside a $m-1$ dimensional hypercube with side length $2\varepsilon/\sqrt{m-1}$ so that
the longest possible distance between its centre and any point in the cube is the half diagonal $\frac{\sqrt{m-1}}{2} \cdot \frac{2\varepsilon}{\sqrt{m-1}}=\varepsilon$. 
Then picking $\hat{w}$ as the centre of the cube will suffice the bounded error $e_2 \leq \varepsilon$.

Therefore, our algorithm runs like a binary search on each of the $m-1$ dimensions. In order to shrink the length of the searching space in each dimension, we recursively select the query $(x,x')$ with difference $\Delta_\phi(x)$ being the normal vector of the hyperplane (approximately) halving the original space along that dimension. Let us call the halving hyperplane $v$. Suppose the width of the version space $\mathcal{W}$ along the current axis is $s.$
By our presumption, we can achieve $ || \Delta_\phi(x)-v||_{\inf}= ||\phi(\Tilde{\phi}^{-1}( \phi(x)+v ))  -(\phi(x)+v) ||_{\inf}< \frac{s}{10}.$

Depending on the noise-free signal, we can cut out at least $\frac{1}{2}-\frac{1}{10}=\frac{2}{5}$ of the space. Repeating the process for the new version space, since we go through the procedure for each dimension, we make $\mathcal{O}(\log(\frac{\sqrt{m-1}}{2\varepsilon})(m-1))$ queries in total. \end{proof}
As before, when the preference labels are corrupted by noise, the problem becomes considerably more challenging, since we can no longer rely on a single label to exhibit an inconsistency among our hypotheses in a give version space.
We address this by repeatedly asking the same query many times, and choosing the majority response.
The resulting algorithm is provided as Algorithm~\ref{alg:noisyal}.

\begin{algorithm}[h]
\caption{Active Learning with Noise}
    \label{alg:noisyal}
    \textbf{Input}: dimension of the instance space $m$, error bound $\varepsilon$ \\
    \textbf{Output}: $\hat{u}$ with $e_2(\hat{u},u) \leq \varepsilon$ 
    \begin{algorithmic}[1] 
    \STATE Initialize the Cartesian coordinate system $\mathcal{C}_1$ for $\mathbb{R}^m$
    \STATE Initialise the version space $\mathcal{W}=\{ w\in \mathbb{R}^{m} \ | \ w>0, ||w||_1=1\}$
    \STATE Initialize a separate $m-1$ dimensional Cartesian coordinate system $\mathcal{C}_2$ for $\mathcal{W}$
    \STATE $hyperplanes \gets []$ 
    \STATE $p_0 \gets F(\varepsilon/\sqrt{m-1})$
    \FOR{$i =1,\dots,m-1$}
    \STATE Let $d$ be the length of the current version space $\mathcal{W}$ along the $i$-th axis of $\mathcal{C}_2$
    \WHILE{$d>2\varepsilon/\sqrt{m-1}$}
    \STATE Let $h_0,h_2$ be two $m-2$-dimensional hyperplanes tangential to the top and bottom of $\mathcal{W}$ along the $i$-th axis of $\mathcal{C}_2$
    \STATE Let $h_1$ be the $m-2$-dimensional hyperplane cutting through the middle of $\mathcal{W}$ along the $i$-th axis of $\mathcal{C}_2$
    \STATE Let $v$ be the outward pointing normal vector of the $m-1$-dimensional hyperplane in $\mathbb{R}^m$ consisting of $h_1$ and the origin of $\mathcal{C}_1$
 \STATE Let $x=0$, $x'=\phi^{-1}(\phi(0)+v)$, and  $h$ be the
        $m-2$ hyperplane in $\mathbb{R}^m$ as the intersection between $\mathcal{W}$ and the $m-1$ hyperplane normal to $\phi(x')-\phi(x)$%
    \FOR{$j =1,\dots,T$}
    \STATE Ask the oracle about $(x,x')$
    \STATE Update $S_T$ accordingly
    \ENDFOR
    \IF {$|S_T-T/2|>T(p_0-1/2)/2$ and $S_T>T/2$}
    \STATE $bounds \gets \{h_0,h\}$
    \STATE $\mathcal{W}\gets$ the half of $\mathcal{W}$ between $h_0$ and $h$
    \STATE $d\gets d/2$
    \ELSIF{$|S_T-T/2|>T(p_0-1/2)/2$ and $S_T<T/2$}
    \STATE $bounds \gets \{h,h_2\}$
    \STATE $\mathcal{W}\gets$ the half of $\mathcal{W}$ between $h$ and $h_2$
    \STATE $d\gets d/2$
    \ELSE
    
    \STATE $hyperplane \gets h$
    \renewcommand{\algorithmicprint}{\textbf{break}}
    \PRINT
    \ENDIF
    \ENDWHILE
    \IF{$hyperplane$ is undefined}
    \STATE $hyperplanes$.add(the hyperplane with equal distance to the bounds)
    \ELSE
    \STATE$hyperplanes$.add($hyperplane$)
    \ENDIF
    \ENDFOR
    \RETURN an intersection point of the whole $hyperplanes$ 
    \end{algorithmic}
\end{algorithm}

We now show that this algorithm is sample-efficient, provided that the noise distribution is ``nice'' in the sense formalized next.
\begin{theorem}
\label{T:active_noise}
    For a fixed $\mathcal{Q}$ with corresponding c.d.f.\ $F$,  suppose we can approximate the inverse of $\phi$ with $\Tilde{\phi}^{-1}$ up to arbitrary precision, i.e. $||\phi(\Tilde{\phi}^{-1}(x))-x||_{\inf} <c$ for any constant $c$, then there exists an active learning algorithm that outputs $\hat{u}$ with $e_2(\hat{u},u)\leq \varepsilon$ with probability at least $1-\delta$ after $\mathcal{O}(\frac{1}{(p_0-1/2)^2} \log(\frac{1}{\delta}))$ queries where $p_0=F(\frac{\varepsilon}{\sqrt{m-1}})>1/2.$ 
\end{theorem}


\begin{proof}
Consider Algorithm ~\ref{alg:noisyal}, the high level procedure is as followed: for every query we constructed in  Algorithm ~\ref{alg:noisefreeal}, we repeat each query $T$ times. If the majority of the answer is not clear, i.e., $|S_T-\frac{T}{2}| \le \frac{T(p_0-1/2)}{2}$ then we know $w$ is close enough to the hyperplane determined by our query with high probability. Otherwise, we opt to trust the majority vote and halve the search space until the distance between our two hyperplanes is smaller than our required length. After looping through every dimension, we complete our hypercube.


Now we start calculating the upper bound of times we need to ask every turn. Let us denote the sum of a query's labels after $T$ repetitions by $S_T.$  Without loss of generality, we assume that the true label of the query on $\Delta_\phi x$ is $1$. Then the expected value $\mathbb{E}(S_T)>T/2$. Additionally, let $p_0$ denote the probability of labels flipping at the face of the hypercube to which $w$'s margin is bounded by $\frac{\varepsilon}{\sqrt{m-1}}.$

First, we would like that if we have witnessed $|S_T-\frac{T}{2}| \le \frac{T(p_0-1/2)}{2}$, with high probability we have already queried a hyperplane with a small margin with respect to $w.$ By Hoeffiding's inequality, we have $\Pr\left(|\mathbb{E}(S_T)-S_T|\ge  \frac{T(p_0-1/2)}{2} \right) \le \exp(-\frac{T(p_0-1/2)^2}{4}).$ So with probability at least $1-\exp(-\frac{T(p_0-1/2)^2}{4})$, $|\mathbb{E}(S_T)-S_T|\le  \frac{T(p_0-1/2)}{2}$. Then the triangle inequality gives us $|\mathbb{E}(S_T) - \frac{T}{2}|= |\mathbb{E}(S_T) - S_T+S_T-\frac{T}{2}|  \leq |\mathbb{E}(S_T)-S_T|+|S_T-\frac{T}{2}|\leq \frac{T(p_0-1/2)}{2}+\frac{T(p_0-1/2)}{2}=T(p_0-\frac{1}{2})$ with probability at least $1-\exp(-\frac{T(p_0-1/2)^2}{4}).$ In other words, with high probability, $T/2 <\mathbb{E}(S_T)\le T(p_0)$. Since $\mathbb{E}(S_T)$ is determined by the distance between the query and $w$, we deduce that      the current query satisfies a small margin as wanted, hence we can stop.

Next, we would like to confirm the majority is the true label with high confidence if we have witnessed a majority vote with significance, i.e.,  when $|S_T-\frac{T}{2}| >\frac{T(p_0-1/2)}{2}$. Recall again with our true label being $1$, $\mathbb{E}(S_T)>\frac{T}{2}.$ The probability of witnessing a false majority is $\Pr\left( S_T-\frac{T}{2} <  -\frac{T(p_0-1/2)}{2}  \right) = \Pr\left( \frac{T}{2}-S_T >  \frac{T(p_0-1/2)}{2}  \right)
\le \Pr\left( \mathbb{E}(S_T)-S_T \ge  \frac{T(p_0-1/2)}{2} \right)
\le \exp(-\frac{T(p_0-1/2)^2}{4})$ since  $\mathbb{E}(S_T)-S_T>\frac{T}{2}-S_T.$


From Theorem \ref{thm:alg-query-bound},  we know we will need $\mathcal{O}(\log(\frac{\sqrt{m-1}}{2\varepsilon})(m-1))$ accurate query labels through votes. So the cumulative confidence is to satisfy $$q^{\mathcal{O}(\log(\frac{\sqrt{m-1}}{2\varepsilon})(m-1))}\geq 1- \delta,$$ where $q$ is the success rate for each query on a different hyperplane.
Therefore $q \ge (1- \delta)^{\frac{1}{\mathcal{O}(\log(\frac{\sqrt{m-1}}{2\varepsilon})(m-1))}}$.
And for each turn we just need
$\exp(-\frac{T(p_0-1/2)^2}{4})\le 1-q$.
By a simple calculation, we get the sufficient number of repetitions from 
$T \ge T_0= -\frac{4}{(p_0-1/2)^2}\log(1-(1- \delta)^{\frac{1}{\mathcal{O}(\log(\frac{\sqrt{m-1}}{2\varepsilon})(m-1))}}).$

So our sample complexity is $\mathcal{O}(T_0 m\log(\frac{\sqrt{m-1}}{2\varepsilon})=\mathcal{O}((\frac{1}{F(\frac{\varepsilon}{\sqrt{m-1}})-1/2)^2} \log(\frac{1}{\delta})))$. \end{proof}


The efficiency result in Algorithm~\ref{alg:noisyal} thus depends on the noise c.d.f.~$F(x)$.
We now show that the condition on the noise in Theorem~\ref{alg:noisyal} obtains for the common case of logistic noise (i.e., the Bradley-Terry model).
Thus, we obtain a sample-efficient active learning algorithm for learning utility parameters in the $\ell_2$ norm when comparison query responses follow the BT model.


\begin{corollary}
    For the standard logistic function $F(x)=\frac{1}{e^{-x}+1}$, the sample complexity is polynomial in $(\frac{1}{\varepsilon},\log (\frac{1}{\delta}),m)$. 
\end{corollary}
\begin{proof}
    Because  for $x=\frac{\varepsilon}{\sqrt{m-1}     
    } \in (0,1), \exp(x) \ge 1 + x$,
    so $\exp(\frac{-\varepsilon}{\sqrt{m-1}     
    }) \le \frac{1}{1+\frac{\varepsilon}{\sqrt{m-1} } }$.
    And $\mathcal{O}(\exp(\frac{-\varepsilon}{\sqrt{m-1}     
    })\log (\frac{1}{\delta}))$ is in 
    $\mathcal{O}((\frac{\sqrt{m} }{ \varepsilon   
    })\log (\frac{1}{\delta})).$
\end{proof}

 Comparing this result with the passive learning result from inequality~(\ref{zhu}), whose sample complexity is $\mathcal{O}(\frac{1}{\varepsilon}\log(\frac{1}{\delta})+\frac{m}{\varepsilon})$,  our sample complexity bound is much better
since for any $m \in \mathbb{N}$ and $0<\varepsilon<1$, $\varepsilon^{\sqrt{m-1}}<\varepsilon< 1< \exp(\varepsilon),$ implying $\frac{1}{\varepsilon}\log(\frac{1}{\delta})> \frac{1}{\exp(\frac{\varepsilon}{\sqrt{m-1}})}\log(\frac{1}{\delta})$.

%% file: appendix.tex
\newpage

\section*{A: Proof of Theorem~\ref{thm:generalisezhu} }

\begin{proof}
    Recall that $\Pr ( x' \succ x ) = F(w^{*^T}\Delta_\phi(x))= 1- F(-w^{*^T}\Delta_\phi(x))= 1-\Pr (x \succ x')$.
    First, we show the strong convexity of the loss function \begin{equation*}
\resizebox{\hsize}{!}{$\ell(w) = -\frac{1}{n}\sum_{i=1}^n \log \left(1(y^i=1) \cdot \Pr(x'\succ x) +  1(y^i=0) \cdot \Pr(x'\prec x)\right)$}
\end{equation*}\begin{equation*}
  \resizebox{\hsize}{!}{$=-\sum_{i=1}^n \log\left(1(y^i=1) \cdot F(w^T\Delta_\phi(x)) +  1(y^i=0) \cdot F(-w^T\Delta_\phi(x))\right).$}  
\end{equation*}
Its gradient and Hessian are     
\begin{equation*}
  \resizebox{\hsize}{!}{$\nabla \ell(w) = -\frac{1}{n}\left[\sum_{i=1}^n \left(1(y^i=1) \cdot \frac{F'(w^T\Delta_\phi(x))}{F(w^T\Delta_\phi(x))} + 1(y^i=0) \cdot \frac{F'(-w^T\Delta_\phi(x))}{F(-w^T\Delta_\phi(x))}\right)\right] \Delta_\phi(x) $}  
\end{equation*}

\begin{equation*}
    \resizebox{\hsize}{!}{$\nabla^2 \ell(w) = \frac{1}{n} \sum_{i=1}^n ( 1(y^i=1) \cdot \frac{F'(w^T\Delta_\phi(x))^2 - F''(w^T\Delta_\phi(x)) \cdot F(w^T\Delta_\phi(x))}{F(w^T\Delta_\phi(x))^2} $}
\end{equation*}
\begin{equation*}
    \resizebox{\hsize}{!}{$+ 1(y^i=0) \cdot \frac{F'(w^T\Delta_\phi(x))^2 - F''(-w^T\Delta_\phi(x)) \cdot F(-w^T\Delta_\phi(x))}{F(-w^T\Delta_\phi(x))^2}) \cdot \Delta_\phi(x) \Delta_\phi(x)^T$}
\end{equation*}

By assumption, we have $F'(z)^2 - F''(z) \cdot F(z) \ge \gamma>0$.
Then we can derive strong convexity of $\ell$:

\[
v^T \nabla^2 \ell(w) v \ge \frac{\gamma}{n} ||Xv||_2^2 \quad \text{for all } v,
\]
where $X$ has $\Delta_\phi(x_i)$ as its $i$-th row,
yielding
\[
\ell(\hat{w}) - \ell(w^*)  - \langle  \nabla \ell(w^*), \hat{w}-w   \rangle   \ge \frac{\gamma}{n} ||X (\hat{w}-w^* )  ||_2^2 = \gamma ||\hat{w}-w^*|| ^2_{\Sigma}.              
\]

Since $\hat{w}$ is the optimal for $\ell$,
\[
  \ell(\hat{w}) - \ell(w^*)  - \langle  \nabla \ell(w^*), \hat{w}-w^*   \rangle  \le -  \langle  \nabla \ell(w^*), \hat{w}-w^*   \rangle.\]

Then as 
\[|   \langle  \nabla \ell(w^*), \hat{w}-w^*   \rangle        | \le ||\nabla\ell(w^*)||_{(\Sigma+\lambda I)^{-1}} ||\hat{w}-w^*||_{\Sigma+\lambda I},      
\] we now would like to bound the term $||\nabla\ell(w^*)||_{(\Sigma+\lambda I)^{-1}} .$

Note the gradient $\nabla \ell(w^*)$ can be viewed as a random vector $V\in \mathbb{R}^n$ with independent component:

\[
V_i=\begin{cases}
 \frac{F'(w^{*^T}\Delta_\phi(x_i))}{F(w^{*^T}\Delta_\phi(x_i))}  & \text{w.p.} ~ F(w^{*^T}\Delta_\phi(x_i))\\
 \frac{F'(-w^{*^T}\Delta_\phi(x))}{F(-w^{*^T}\Delta_\phi(x_i))} & \text{w.p.} ~ F(-w^{*^T}\Delta_\phi(x_i))
\end{cases}.
\]

We know that $F(w^T\Delta_\phi(x_i))+ F(-w^T\Delta_\phi(x_i))= \Pr(x' \succ x)+\Pr(x \succ x')=1 $, taking the derivative of this equation, we can conclude that the expected value is zero $\mathbb{E}[V]= F'(w^T\Delta_\phi(x_i))+F'(-w^T\Delta_\phi(x))=0 $.

And because $F$ is the c.d.f., and $F'$ is the p.d.f, by definition $F'(z)<F(z)$, and 
$ \frac{F'(w^T\Delta_\phi(x_i))}{F(w^T\Delta_\phi(x_i))}<1$, $\frac{F'(-w^T\Delta_\phi(x))}{F(-w^T\Delta_\phi(x_i))}<1.$

Hence all the variables $V_i$ are $1$-sub-Gaussian, and the Bernstein’s inequality in quadratic form applies (see e.g. \cite{hsu2012tail} (Theorem 2.1)) implies that with probability at least $1 -\delta$,

\[
||\nabla\ell(w^*)||^2_{(\Sigma+\lambda I)^{-1}} =V^T M V \le C_1 \cdot \frac{d+\log(1/\delta)}{n},
\]

where $C_1$ is some universal constant and $M= \frac{1}{n^2} X(\Sigma+\lambda I)^{-1}X^T$ .

Furthermore, let the eigenvalue decomposition of $X^T X$ be $U\Lambda U^T.$
Then we can bound the trace and operator norm of $M$ as

$Tr(M)=\frac{1}{n^2}Tr(U(\Lambda/n+\lambda I)^{-1}U^{T}U \Lambda U^T) \le \frac{d}{n}$
\begin{equation*}
      \resizebox{\hsize}{!}{$Tr(M^2)=\frac{1}{n^4}Tr(U(\Lambda/n+\lambda I)^{-1}U^{T}U \Lambda U^T) U(\Lambda/n+\lambda I)^{-1}
U^{T}U \Lambda U^T)\le \frac{d}{n^2}$}
\end{equation*}
$||M||_{op}=\lambda_{max}(M)\le \frac{1}{n}$

 This gives us

\[
\gamma ||\hat{w}-w^*|| ^2_{\Sigma} \le   ||\nabla\ell(w^*)||_{(\Sigma+\lambda I)^{-1}} ||\hat{w}-w^*|| ^2_{\Sigma}                  
\]

\[
\le \sqrt{C_1 \cdot \frac{d+\log(1/\delta)}{n}} ||\hat{w}-w^*|| ^2_{\Sigma}                         
\]

Solving the above inequality gives us,

\[
||\hat{w}-w^*||_{\Sigma} \le C_2 \cdot\sqrt{  \frac{d+\log(1/\delta)}{n}} 
\]

\end{proof}